\documentclass{article} 
\usepackage{iclr2017_conference,times}
\usepackage{hyperref}
\usepackage{url}

\usepackage{times}
\usepackage{epsfig}
\usepackage{graphicx}
\usepackage{amsmath}
\usepackage{amssymb}

\usepackage{multirow}
\usepackage{tabularx}
\usepackage{dsfont}
\usepackage{url}
\usepackage{subfigure}
\usepackage{amsthm}
\usepackage[ruled,vlined]{algorithm2e}
\usepackage[export]{adjustbox}

\pdfoutput=1

\newtheorem{theorem}{Theorem}[section]
\newtheorem{lemma}[theorem]{Lemma}

\newtheorem{property}[theorem]{Property}

\theoremstyle{definition}
\newtheorem{definition}{Definition}[section]

\theoremstyle{remark}

\title{How ConvNets model Non-linear Transformations}

\author{Dipan K.~Pal \& Marios Savvides \\
Department of Electrical and Computer Engineering\\
Carnegie Mellon University\\
Pittsburgh, PA 15213, USA \\
\texttt{\{dipanp,msavvid\}@cmu.edu} \\
}

\begin{document}

\maketitle

\begin{abstract}
In this paper, we theoretically address three fundamental problems involving deep convolutional networks regarding invariance, depth and hierarchy. We introduce the paradigm of Transformation Networks (TN) which are a direct generalization of Convolutional Networks (ConvNets). Theoretically, we show that TNs (and thereby ConvNets) are can be invariant to non-linear transformations of the input despite pooling over mere local translations. Our analysis provides clear insights into the increase in invariance with depth in these networks. Deeper networks are able to model much richer classes of transformations. We also find that a hierarchical architecture allows the network to generate invariance much more efficiently than a non-hierarchical  network. Our results provide useful insight into these three fundamental problems in deep learning using ConvNets.
\end{abstract}

\section{Introduction}

It is a well known fact that deep Convolutional Networks (or ConvNets) \cite{lecun1998gradient} generate invariance to local translations due to convolutions followed by a form of pooling. In practice, however, studies such as \cite{krizhevsky2012imagenet} have applied these models very successfully to domains such as vision, which typically involve data undergoing highly non-linear transformations. It is therefore clear, that these models can model invariance towards these global non-linear transformations despite solely employing pooling over local translations. Further, \cite{simonyan2014very} observed that a deeper ConvNet usually performs better (and thus is more invariant) on large scale tasks. This raises some fundamental questions. 

\textbf{Problem 1:}\textit{ How does a ConvNet generate invariance to global non-linear transformations through pooling over mere local translations?}

\textbf{Problem 2:}\textit{ How does invariance increase with depth in ConvNets?}

\textbf{Problem 3:}\textit{ How does a hierarchical architecture help?}


These have been long standing problems in vision since the inception of these networks. Intuitions and empirical observations abound, the problems still are not completely addressed from a theoretical standpoint. 

\textbf{Main results:} In this paper, we take a significant step towards answering these questions.

\textit{Addressing Problem 1:} We show that these non-linear invariances arise from the \textit{architecture} of the network itself rather than the exact features learnt. More specifically, the entire pipeline of convolution followed by pooling and then a non-linearity itself contributes towards learning such powerful invariances. Although optimizing the features is important to capture the most amount of ``information" and provide descriptive features, \textit{invariance} strictly speaking, is not generated due to the features themselves. Instead, it is a by-product of the architecture. Our main result shows that a $L$ layered ConvNet (and also a generalization of such architectures introduced as Transformation Networks or TNs), generates invariance to transformations $h(x)$ of the input $x$ of the form $h(x) =  g_1\circ \eta \circ g_2 ...\eta \circ g_L(x)$\footnote{ $g \circ \eta (x) = g(\eta (x))$} where $g_i$ is a unitary transformation and $\eta$ is a point-wise applied non-linearity satisfying certain conditions of unitarity and stability. A very good approximation of such a non-linearity is the hard-ReLU which is prevalent in practice  \cite{nair2010rectified}, thereby providing a theoretical justification of the same. The form of $h(x)$ transformation highly non-linear. Even though unitary transforms include commonly known and ``elementary" transforms such as translation and in-plane rotation, their composition with $L-1$ non-linearities make the overall transformation very rich and powerful. 

\textit{Addressing Problem 2:} Further, it immediately shows why depth is an important parameter in ConvNet architecture design. Increasing $L$ in our model allows us to be invariant to a more expressive transformation form. Loosely speaking, each layer of the ConvNet can be said to generate invariance to one pair of $g_i$ and $\eta$. The precise form of $h(\cdot )$ depends on the exact hierarchy employed by the architecture and is discussed in more detail in a later section. The architecture of a ConvNet itself is a form of incorporating a prior on the kind of  nuisance transformations expected to be observed in the data. This is complimentary to the regularization implications of weight sharing. 

\textit{Addressing Problem 3:} We also show that the hierarchical nature of a ConvNet also helps in significantly improving efficiency in generating invariance. A $L$ layered ConvNet reduces the number of required observations of transformed inputs for training from $\mathcal{O}(|\mathcal{G}|^L)$ to $\mathcal{O}(|\mathcal{G}|)$, a reduction of the order $L$. 


\textbf{Intuitive Proof Sketch:} We first prove that each node at the first layer of a ConvNet (also Transformation Networks) generates invariances towards or factors out local translations (and more general unitary transforms for TNs). Then we put two conditions (unitarity and stability) on the point-wise non-linearity  used in these networks such that transformations that were not factored out in the first layer are propagated to the second layer. We find that a the implicit mapping of a fractional degree polynomial kernel exactly satisfies unitarity and very closely approximates stability for a well chosen range of degrees. This function is also a very close approximation of the hard-ReLU non-linearity. The non-linearity helps preserve the group structure of the transformed inputs in the feature space. We finally show that every second layer node then is able to generate invariance to the left over transformations (not captured in the first layer) even if they had acted on the input after a non-linearity. This way the second layer node is overall invariant to a non-linear transformation of the input. As we go up passing through more layers, they add in abilities to be invariant to more non-linearities and complexities.

\textbf{Prior Art: }Deep learning despite its great success in learning useful representations, has yet to have a very concrete theoretical foundation. Nonetheless, there have been many attempts at a deeper understanding of its mechanics. For instance, \cite{kawaguchi2016deep}  proved important results for deep neural networks. Whereas \cite{CohenSS15a, HaeffeleV15} approached deep learning from the perspective of general tensor decompositions. All of these studies however, have focused on the supervised version of deep learning. Under supervision, theoretical results can be broadly described to be concerned with the optimality of a solution or properties of the optimization landscape. Given the success of supervised models, such an approach is definitely beneficial in advancing overall understanding. It however, considers architectures more general in nature since supervised results for specialized architectures are more difficult to obtain.

Unsupervised deep learning however, promises to play an important role in the future not to mention kindling interests from a neoroscientific perspective.  The analysis of our models is therefore aimed at the unsupervised setting and focuses more on the invariance properties of such networks. This reveals new insights into properties of the architecture itself and provides an explanation as to why increasing depth is useful on many fronts. Even though there have been theoretical efforts \cite{delalleau2011shallow, martens2014expressive} to provide results related to the ``depth" of a network, the models studied do not immediately resemble the most successful architecture class in practice, ConvNets and its variants. We present results on a generalization of ConvNets called Transformation Networks (TN) which are directly applicable to ConvNets.  In fact, TNs are very closely related to ConvNets and become identical under a very simple constraint. 

There have been a few important efforts towards providing results from a unsupervised standpoint \cite{anselmi2013unsupervised, mallat2012group}. \cite{mallat2012group} shows that local translation invariance leads to contractions in space. However, it is not clear whether those contractions are due to \textit{non-linear} invariances. \cite{anselmi2013unsupervised} approach the problem in a fashion more similar to ours with the use of unitary groups to ``transfer" invariance. They show that for a hierarchical feed-forward network with unitary group structure, the features at top layers would be exactly invariant to groups of transformations acting over a larger receptive field. Our main result, on the other hand is more precise. We show that the top layer features is in fact invariant to \textit{non-linear} transformations \textit{despite} only pooling over \textit{linear} transforms. Further, these non-linear transformations need not form a group overall. They are only required to form a group locally at every layer. The architecture we consider is very closely related to practical architectures used for ConvNets, whereas \cite{anselmi2013unsupervised} model the architecture utilizing simple and complex cell constructions from a more biologically motivated approach. Further, they hypothesize that the non-linearity serves as a way measuring bins of the CDF of an invariant distribution. On the other hand, we consider the non-linearity to be an integral part of the process to preserve unitary group structure in the feature space. This also leads to it being a part of the class or range of transformations to be invariant towards. In turn this observation leads to the critical result that the overall architecture is invariant to \textit{non-linear transformations despite pooling over linear transforms}. 

Finally, \cite{bruna2013learning, paul2014does} also applied group theory to a certain extent to the problem of representation learning. These works provide useful insights into stabilization with groups. Here, stabilization is meant along the lines of resulting in a contraction or non-expansion of the space. Nonetheless, they do not explore exact invariance to explicitly non-linear transforms as our study.

\section{Transformation Networks}

We  introduce the paradigm of Transformation Networks (TN), a more general way of looking at feed forward architectures such as ConvNets and present results on these and then apply them directly to ConvNets. We first briefly review the notion of unitary groups and group invariant functions.

\textbf{Premise and Notations:}  We denote images and general vectors by $x\in \mathbb{R}^d$. Given such a $x$, we define a support set $\Lambda$ which defines a subset of pixels or dimensions over $x$, \emph{i.e.} $x_{\Lambda}$ defines the subset of pixels contained in the set of indices $\Lambda$ arranged in a column. Given an image $x$, we consider it divided it into small non-overlapping regions covering the entire image. Each support set is denoted by $\Lambda_{li}$ \emph{i.e.} the $i^{th}$ support set at layer $l$ as shown in Fig.~\ref{fig_TN_1}. $\Lambda_{li}$ is a union of certain $\Lambda_{(l-1)j}$ as defined by a hierarchy (say in a ConvNet).  For instance in Fig.~\ref{fig_TN_1}, $\Lambda_{2}$ (shaded light blue) is the union of the supports $\Lambda_{11}, \Lambda_{12}, \Lambda_{13}, \Lambda_{14}$ in the image plane. This union of support is similar to the hierarchical structure observed in ConvNets and is defined by the specific architecture.

\textbf{Unitary-Group:} A group $\mathcal{G}$ is a set of elements $g\in \mathcal{G}$ along with the properties of closure, associativity, invertibility and identity \footnote{We will mostly deal with continuous groups however, our results also hold for discrete groups.}. A unitary group is any group whose elements are unitary in nature, \emph{i.e.} the dot-product is preserved under the unitary transformation. More precisely, $\langle g(x), g(y)   \rangle = \langle  x, y \rangle~~\forall x, y$. $g(x)$ denotes the action of the group element (or transformation) $g$ on $x$. The action of a group can also be constrained by $\Lambda$. For instance, $g_\Lambda$ is a unitary transform acting only on the support set $\Lambda$. We express the  action of a transform on a restricted support by $g_\Lambda(x)$.

\begin{figure}
\centering
\includegraphics[width=0.8\columnwidth,valign=m]{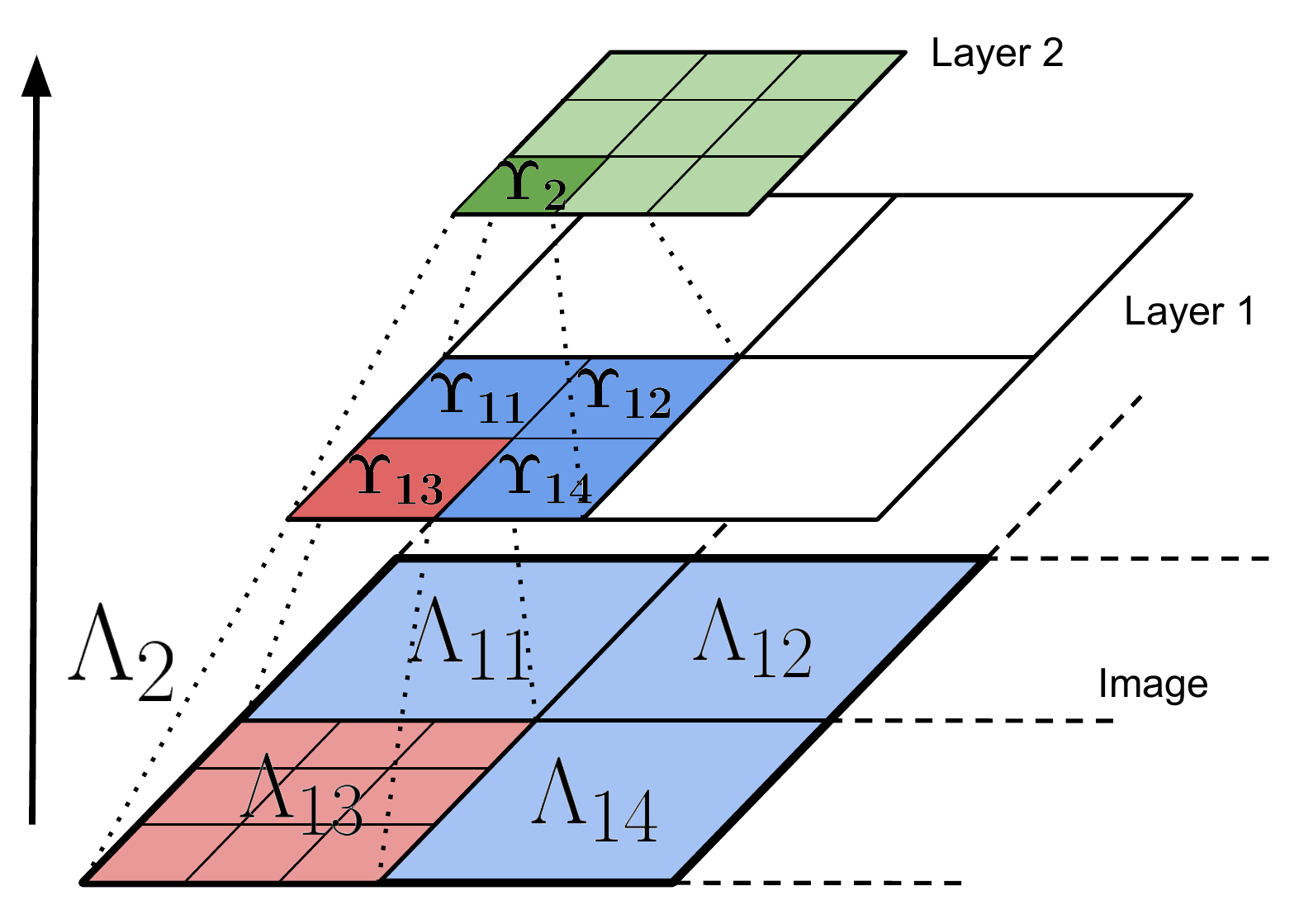}
\centering 
\caption{ Structure of a two layered Transformation Network (TN). Node $\Upsilon_{13}$ (bold red) of layer 1 has a receptive field highlighted in light red. Node $\Upsilon_2$ of layer 2 (bold green) has a receptive field highlighted in light blue (also contains node $\Upsilon_{13}$'s receptive field). }
\label{fig_TN_1}
\end{figure}

\textbf{The Unitary Non-linear Image Transformation Model:} Unitary groups are very useful in modelling linear transformations in domains such as images. Indeed, translation and in-plane rotation can be modelled as unitary and expressed as $g(x)$. However, coupled with a non-linearity $\eta(\cdot)$ and restricted support on the image $\Lambda$, unitary transforms can model a far richer class of images. Let $\mathcal{X} = \{ g(x) ~|~g\in \mathcal{G}  \}$ be the set of all transformations of $x$ generated $\mathcal{G}$. Now, for a given non-linearity $\eta(\cdot)$, consider a non-linear transformation as $g_{\Lambda_{11}}g_{\Lambda_{12}}g_{\Lambda_{13}}g_{\Lambda_{14}}(\eta( g_{\Lambda_{2}}(x)  ))$. Here $g_{\Lambda_{11}}g_{\Lambda_{12}}$ apply the individual transforms over the specified support. Notice that $g_{\Lambda_{11}}...g_{\Lambda_{14}}$ are jointly unitary. This is because each $g_{\Lambda_{1i}}$ is a unitary transformation over the support $\Lambda_{1i}$, and $\bigcap_i \Lambda_{1i} =0$, \emph{i.e.} the supports are non-overlapping. Lastly, $\Lambda_2$ is a union of $\Lambda_{11}..\Lambda_{14}$ and $g_{\Lambda_2}$ is a unitary transformation over a larger support. This expression of a non-linear transformation of $x$ is more powerful than the simply linear $g(x)$ primarily due to the non-linearity $\eta$, thereby allowing the modelling of much richer variation in data.


\textbf{Transformation Networks (TN):} Transformation Networks (TN) are essentially feed forward networks that operate primarily on the principle of generating invariance towards a group or set of transformations through pooling modelled as group integration. The architecture of of these networks are hierarchical in nature and they \textit{explicitly} invoke invariances only locally and can potentially have multiple layers. In doing so, they \textit{implicitly} can model global invariances. Consider a TN with $L$ layers. Each layer has a number of TN nodes each with a receptive field size of $(w_l, h_l)$, \emph{i.e.} each cell or node in the layer can only look at patches of size $w_l \times  h_l$ of the output from the previous layer. Every node at layer $l$ can take in a number of input channels $o_{l-1}$ from the previous layer, and output a number of channels $o_l$ to the next layer. Further, each node has a set of filters or templates $\mathcal{T}_{li} = \{ g(t_i) ~|~g\in \mathcal{G}_{li} \}, \forall i=1,..,o_l$ of size $w_l \times  h_l$. Here $\mathcal{G}_{li}$ is any unitary group specific to the $i^{th}$ output of the $l^{th}$ layer. We call $\mathcal{T}_{li}$ as a \textit{template set} (henceforth to be assumed under some specified $\mathcal{G}_{li}$). The template set simply a set of templates transformed under the action of $\mathcal{G}_{li}$. Thus, there are $o_l$ such transformation blocks in layer $l$. Every node contains a pooling operation which performs group integration over the template set (essentially mean pooling). Further, there is a point-wise non-linearity applied to the pooled feature.


\textbf{TN Node:} A TN node $\Upsilon_{li}$ (the $i^{th}$ node at layer $l$) provides a single dimensional feature given a patch $x$ of size $(w_l, h_l)$. The node output, for a given non-linearity $\eta$ and input $x$, is given by
\begin{align}
  \Upsilon_{li}(x) &= \eta(  \int_{\mathcal{G}_{li}} \langle  x, g(t_i)  \rangle dg   ) \label{TN_eq_1}\\
  &\approxeq \eta( \frac{1}{|\mathcal{G}_{li}|} \sum_{\mathcal{G}_{li}}  \langle  x, g(t_i) \rangle) \label{TN_eq_2}
\end{align}
Here, recall that $\mathcal{G}_{li}$ is a unitary group and $t_i$ is the template for that particular node. Note that Equation~\ref{TN_eq_2} models an average pooled ConvNet exactly for $\eta$ being the hard-ReLU function and $\mathcal{G}_{li}$ being the translation group.  However, the results for the TN node also hold for max pooling. Equation~\ref{TN_eq_2}  is the version in which the group is a discrete finite group. All results also hold for the discrete case. Fig.~\ref{fig_TN_2} illustrates a single channeled TN node observing two support sets.

\begin{figure}
\centering
\includegraphics[width=0.8\columnwidth,valign=m]{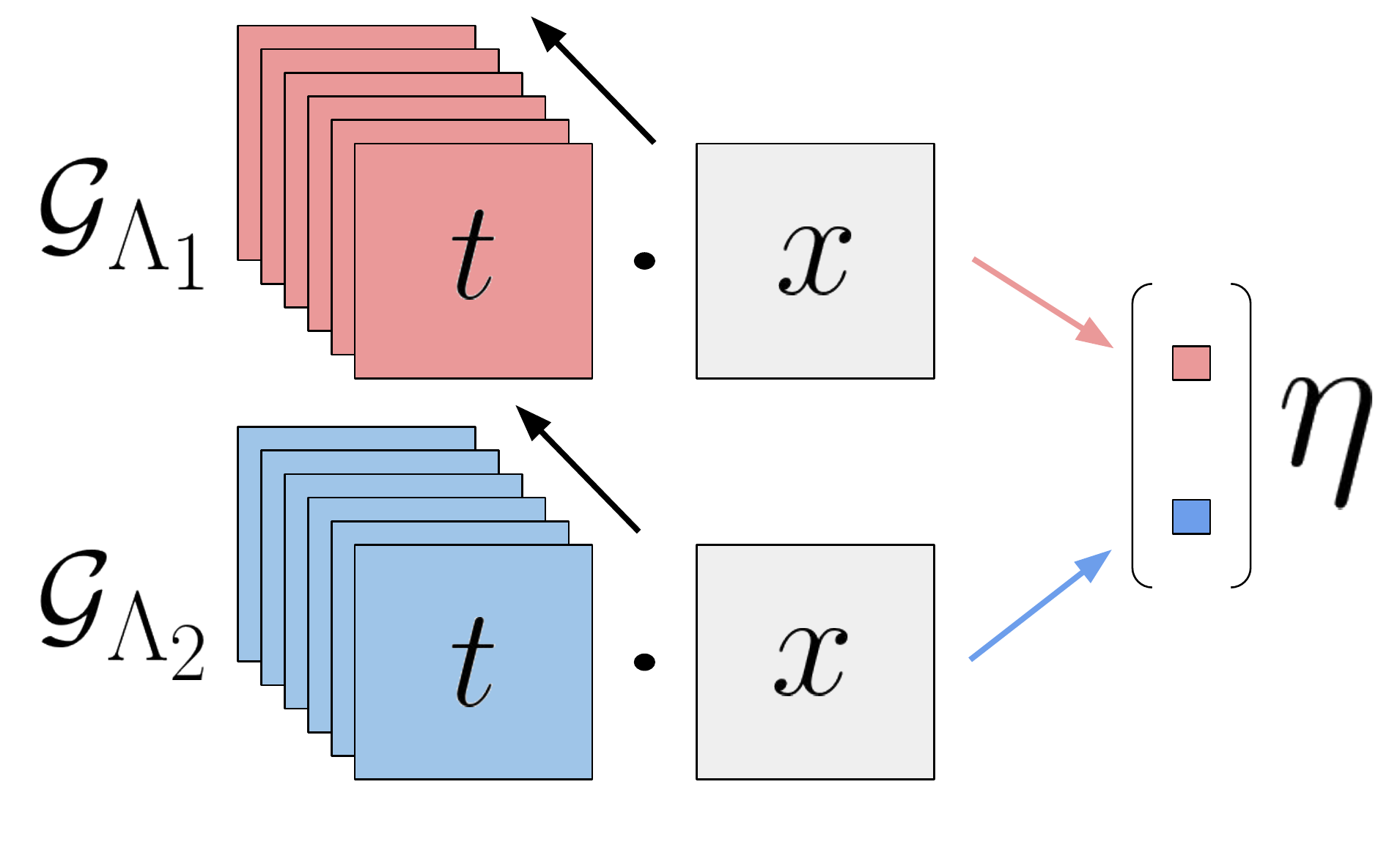}
\centering 
\caption{A single channeled TN node observing two receptive fields denoted by $\Lambda_1, \Lambda_2$. The templates from each support (potentially different) transform according to the groups $\mathcal{G}_{\Lambda_1}, \mathcal{G}_{\Lambda_2}$ (which potentially could be identical, as in ConvNets). The input to the node $x$ has a dot-product with all transformed templates and then integrated over (pooling over the black arrow). Finally, the non-linearity  $\eta(\cdot)$ is applied. }
\label{fig_TN_2}
\end{figure}



\textbf{Learnable Components in a Transformation Network:} The only learnable parameters in a Transformation Network (after the architecture is finalized) are the sets of filters $\mathcal{T}_{li}$  $\forall i$ in each $l$. However, each set has two components to be learnt. 1) The first is the template $t_{li}$, the template for the $i^{th}$ node at layer $l$ (analogous to a feature). 2) The second is the group $\mathcal{G}_{li}$ with which the template $t_{li}$ transforms. Note that once a single template $t_{li}$ is specified along with the corresponding group $\mathcal{G}_{li}$, all transformed templates in the template set $\mathcal{T}_{li}$ are specified. Thus, contrary to convolutional architectures which only learn the filters, Transformation Networks are required to learn both the transformations \textit{and} the filters. Though main focus of this paper are the invariance properties of these networks, we briefly investigate how one could learn a Transformation Network.

\textbf{Unsupervised Learning of a Transformation Network:} In the unsupervised setting, TNs can be trained in a greedy layer-by-layer fashion. The training data is passed through layer 1 of the TN to learn the templates $t_{li}$ and the corresponding transformation groups $\mathcal{G}_{li}$ at the same time. One simple way is to sample the transforming input sequence. Doing so specifies both the templates and the corresponding groups simultaneously. Unsupervised feature learning techniques such as ICA can also be applied. Once layer 1 is trained, layer 1 features can be extracted from the training data before being passed to layer 2 for training the second layer. This process can be repeated until all layers are trained.

\textbf{Supervised Learning of a Transformation Network:} Under the supervised setting, one can assume that gradients are available. It is harder to train under this setting since the gradients need to update each template set or transformation block while keeping its group structure intact. One way of addressing this issue is to assume a particular group structure throughout the TN. This is the exact assumption that ConvNets make. ConvNets model all transformation groups in the network as the translation group which is parametric. The parametric nature allows one to compute the transformed template set on the fly. Thereby the only learnable parameters are the initial templates or filters $t_{li}$. This brings us to the realization that a TN modelling general groups might model invariances better than a ConvNet, an  observation we explore more in the following section. Nonetheless, our main result shows that ConvNets (and TNs in general) can in fact model non-linear invariances.

\section{Invariances in a Transformation Network}
\subsection{Linear Unitary Group Invariance in single layer Transformation Networks} 

We will show that a single layered TN, more specifically a single TN node,  $\Upsilon(x)$ can be invariant to any unitary group $\mathcal{G}$ in the following sense.

\begin{definition}[\textit{$\mathcal{G}$-Invariant Function}]\label{def_invariant}
For any group $\mathcal{G}$, we define a function $f:\mathcal{X} \rightarrow \mathbb{R}^n$ to be $\mathcal{G}$-invariant if $f(x) = f(g(x))~\forall x\in \mathcal{X} ~\forall g\in \mathcal{G}$.
\end{definition}

An invariant to any group $\mathcal{G}$ can be generated through the following (previously) known property utilizing group integration. This is a basic property of groups and arises due to the invariance of the Haar measure $dg$ \footnote{Proof in the supplementary.}.


\begin{lemma}\label{lem_invariance} (Invariance Property) Given a vector $x\in \mathbb{R}^d$, and any group $\mathcal{G}$, for any fixed $ g' \in \mathcal{G}$ and a normalized Haar measure $dg$, the following is true $g'\left(\int_\mathcal{G} g(x)\right) ~dg = \int_\mathcal{G} g (x) ~dg$
\end{lemma}

\textbf{One layer TN is invariant to unitary transformation groups in the input space:} Consider a TN with just a single layer of TN nodes. Each of these nodes looks at a patch of the same size. Each output feature of the network is given by Eq.~\ref{TN_eq_2}, although to study the properties of such a construction, we will utilize Eq.~\ref{TN_eq_1}. Utilizing Lemma~\ref{lem_invariance} along with the definition of a TN node, we have the following.

\begin{lemma}\label{lem_TN_node_invariance}  (TN node linear Invariance) Under a unitary group $\mathcal{G}$, under the action of which the filters or templates $\mathcal{T}$ of a TN node are transformed, the node output is invariant to the action of $\mathcal{G}$ on the input $x$, \emph{i.e.} $\Upsilon(x) = \Upsilon(g'(x)) ~~\forall g' \in \mathcal{G}, \forall x$.
\end{lemma}

The proof is provided in the supplementary.  This result shows that the TN node is invariant to local \textit{linear} transformations (locality depending on the size of the receptive field). There are two main properties of the unitary group which allow for such invariance of the input. First, the group structure itself allows for invariant to be computed through group integration. Secondly, the unitary property of each element allows for the transformation to be ``transferred" from the template $t$ to the input $x$ \emph{i.e.} $\langle x, g(t)   \rangle = \langle g^{-1}(x), t   \rangle$. Thus, \textit{integrating over $g(\cdot)$ is equivalent whether we compute this over input or the template}. Transformation Networks compute this integration over the pre-transformed templates, thereby computing an invariant feature of $x$ even though it has never observed any other transformation of $x$.  The unitarity of the transformations allows us to be invariant to the transformed versions of the input even though we might have never observed them in training. 


In the following section, we show that under certain conditions that are very closely approximated in practice, exact invariance can be achieved to non-linear transformations of the input as well. This is a fundamental problem in generalized (supervised and unsupervised) deep learning. Specifically, how does a deep feed-forward network generate invariance to the highly non-linear transformations in data? Much of the attention for the answer to this question has gone to learnable features. We find that the inherent structure of the network itself (such as in ConvNets) is ideal to invoke invariance. In our group theoretic framework, these "features" or filter weights would  be the point from which the transformed filters are generated \emph{i.e.} $x$ in $g(x)$. 


\subsection{Non-linear Activation in Transformation Networks}
In the case of a TN with 2 or more layers, the non-linear activation function (under certain conditions) can help in generating invariance to \textit{non-linear} transformations in the input space. In order to show this, we first show that under the unitary condition, such a non-linear activation can preserve the unitary group structure in the range space of the function \emph{i.e.} the unitary transformation in the input domain of the non-linear activation function is also a corresponding albeit different unitary transformation in the range of the function. This unitary group structure is observed by TN nodes downstream (higher up the layers), which then through group integration to be able to generate invariance to the same utilizing Lemma~\ref{lem_TN_node_invariance}. 

\begin{figure}
\centering
\includegraphics[width=0.8\columnwidth,valign=m]{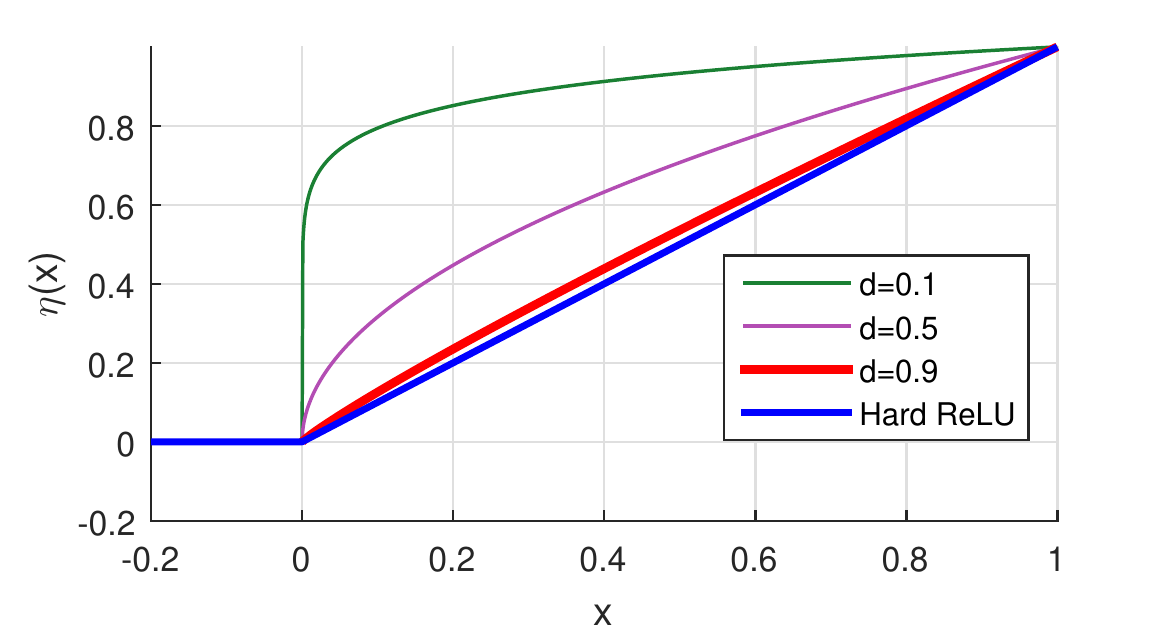}
\centering 
\caption{ Non-linear activation function satisfying conditions of unitarity and stability. Here $\eta(x) = x^d$ is plotted for different values of $d = \{ 0.1, 0.5, 0.9 \}$, however only $0.9 \leq d <1$ approximately satisfy stability. Hard ReLU $=\max(0, x)$, which is prevalent in practical deep learning, is plotted for comparison. In the 1-dimensional case, $\eta(x) = x^{d}$ for $0.9 \leq d <1$ which satisfies Condition 1 (unitarity) exactly and is a very close approximation of Condition 2 (stability) and hard ReLU. }
\label{fig_activation}
\end{figure}

\textbf{Conditions on the non-linear activation function:}  We now state the conditions on the non-linear activation function $\eta(\cdot)$.
\begin{enumerate}
    \item \textbf{Condition 1: }\textit{(Unitarity)} We define a function $\eta: \mathcal{X} \rightarrow \mathcal{H}$ to be a unitary function if, for a unitary group $\mathcal{G}$, it satisfies $\langle \eta(g(x)), \eta(g(y))\rangle = \langle \eta(x), \eta(y)\rangle~\forall g\in \mathcal{G}, \forall x, y \in \mathcal{X}$.\label{def_unitary_kernel}


\item \textbf{Condition 2: }\textit{(Stability)} We define a function $\eta: \mathcal{X} \rightarrow \mathcal{H}$ to be stable if $\eta \circ \eta (x) = \eta( \eta( x ) ) = \eta(x) ~~ \forall x \in \mathcal{X}$.\label{def_stable_function}

\end{enumerate}


Many functions prevalent in machine learning are unitary in the sense of Condition~\ref{def_unitary_kernel}. One example is the class of polynomial kernels $k(x, y) = (x^Ty + c)^d$. Since the kernel employs an actual dot-product, it is clear that the function is unitary. The activation function of interest $\eta(\cdot)$ in this case would be the non-linear implicit map that the kernel defines from the input space to the Reproducing Kernel Hilbert Space (RKHS) (\emph{i.e.} $k(x,y) = \langle \eta(x), \eta(y)  \rangle$). For an example of a function that is stable in the sense of Condition~\ref{def_stable_function}, we consider Rectified Linear Units or the hard ReLU activation function ($\max(0,x)$) which is prevalent in deep learning \cite{nair2010rectified}. Note that both conditions of unitarity (Condition 1) and stability (Condition 2) need to be satisfied by the activation function $\eta(\cdot)$. We find such a class of non-linear functions in the implicit kernel map of the polynomial kernel $k(x,y) = \langle \eta(x), \eta(y)  \rangle = \langle x, y  \rangle^d$ with $d$ strictly less than but close to 1 \emph{i.e.} $d\rightarrow 1, d \neq 1$. Although not prevalent, such kernels are valid \cite{rossius1998short}. These functions exactly unitary and are approximately stable ($d$ being arbitrarily close to 1 but not equal) for the range of values typical in activation functions.  For the 1-D case with $d=0.9$, $\eta(x) = x^{0.9}$. Restricting the function to produce only real values, it rejects all negative values in its domain. This behavior is a very close approximation of the hard rectified linear unit \emph{i.e.} $\max(0, x)$ as illustrated in Fig.~\ref{fig_activation}.

\textbf{Group structure is preserved in the range of $\eta(\cdot)$:} One of our central results is that group invariance can be invoked through group integration in the non-linear feature space as well. This is the crux of the invariance generation capabilities of ConvNets and Transformation Networks in general. 

We define an operator $g_\eta:\eta(x) \rightarrow \eta(g(x))$ for any $g\in \mathcal{G}$ where $\mathcal{G}$ is unitary. $g_\eta$ is thus a mapping within $\mathcal{H}$ (the range of $\eta(\cdot)$). Under unitary $\mathcal{G}$, we then have the following result.
\begin{theorem}\label{theorem_1}
(Covariance in the range of $\eta(\cdot)$) If $\eta(\cdot) $ is a unitary function in the sense of Definition~\ref{def_unitary_kernel}, then  $g_\eta$ is unitary, and the set $\mathcal{G}_\eta = \{ g_\eta ~|~ g_\eta: \eta(x)\rightarrow \eta(g(x))~\forall g\in \mathcal{G} \}$ is a unitary-group in $\mathcal{H}$. This implies $\eta(g(x)) = g_\eta (\eta(x)) ~~\forall x$ with $g_\eta$ being unitary.
\end{theorem}

Theorem~\ref{theorem_1} shows that the unitary group transformation in the input space of a TN node can be expressed as a unitary group transformation in the range or feature space of the node \footnote{ Proof in the supplementary.}. Since the group structure is preserved in the non-linear space, unitary group integration allows for transformation invariance of the input.

\subsection{Non-linear Group Invariance in multi-layer Transformation Networks}

\textbf{Analysis of a two-layered TN:} Consider a simple 2 layered TN with four TN nodes in layer 1 $\{ \Upsilon_{11}, \Upsilon_{12}, \Upsilon_{13}, \Upsilon_{14}  \}$ each looking at non-overlapping patches of the raw input image. Let one TN node $\Upsilon_{2}$ at layer 2 be receiving the spatially concatenated input from layer 1 as shown in Fig.~\ref{fig_TN_1}. Let all nodes be single channel nodes with the templates and their corresponding unitary groups as $\{(t_{11}, \mathcal{G}_{11}), (t_{12}, \mathcal{G}_{12}), (t_{13}, \mathcal{G}_{13}), (t_{14}, \mathcal{G}_{14}), (t_{2}, \mathcal{G}_{2}) \}$ \footnote{ Weight sharing would imply the same templates and groups. Our results are not affected by that constraint.}. The transformed templates were learnt according to the unsupervised learning protocol or according to a supervised learning protocol which \textit{preserves the group structure of each template set}. Since layer 1 TN nodes are already invariant to unitary groups $\{ \mathcal{G}_{11}, \mathcal{G}_{12}, \mathcal{G}_{13}, \mathcal{G}_{14}  \}$, the only transformation that layer 2 nodes might observe are the ones that are not captured by these groups. $ \mathcal{G}_{2}$ is a transformation group that is unitary over the support $\Lambda_2$ which has a receptive field that is a union of $\{ \Lambda_{11},\Lambda_{12},\Lambda_{13},\Lambda_{14}\}$. $\mathcal{G}_2$ is \textit{not} necessarily unitary over the individual supports (receptive field defined by $\Lambda_{1j}$) of layer 1 nodes.


The output of the layer 1 nodes for a single channel with templates $\{ t_{1i} \}$ are of the form $\Upsilon_{1i}(x_{\Lambda_{1i}}) = \eta(  \int_{\mathcal{G}_{1i}} \langle  x_{\Lambda_{1i}}, g(t_{1i})  \rangle dg)  =  \eta(  I(x_{\Lambda_{1i}}) ) $, where we replaced the group integral over the dot-product with $I(\cdot)$ to emphasize an invariant feature. The output for the layer 2 node $\Upsilon_2$ for a single channel with template $t_2$ is as shown below. $$\Upsilon_2( x ) =  \eta\left( \int_{\mathcal{G}_2}  \langle x,  \eta(g(t_2))  \rangle dg \right)$$ Note that since the templates are learnt in an unsupervised fashion by passing in input through the previous layers, the transformed template $t_2$ is of the form $\eta(g(t_2))$. This non-linearity appears due to the non-linear activation function of the previous layer. Even in the case of learning by back-propagation, weights pass through multiple non-linearities resulting in similar forms for the templates. For a two layered TN network, the second layer TN node expression is of the following form. $$\Upsilon_2( o_1 ) =  \eta\left( \int_{\mathcal{G}_2}  \langle \left[ \begin{array}{c} \eta(I(x_{\Lambda_{11}})) \\ \eta(I(x_{\Lambda_{12}})) \\ \eta(I(x_{\Lambda_{13}})) \\ \eta(I(x_{\Lambda_{14}}))  \end{array} \right],  \eta(g(t_2))  \rangle dg \right)$$where $o_1 = [\Upsilon_{11}(x_{\Lambda_{11}}), \Upsilon_{12}(x_{\Lambda_{12}}),  \Upsilon_{13}(x_{\Lambda_{13}}),  \Upsilon_{14}(x_{\Lambda_{14}})]^T$. Replacing the concatenated invariant feature vector $ [ I(x_{\Lambda_{11}}),  I(x_{\Lambda_{12}}), I(x_{\Lambda_{13}}), I(x_{\Lambda_{14}}) ]^T  = x'_{\Lambda_2}$ and applying the point-wise non-linearity over the entire vector,  the layer 2 feature output becomes,\begin{align}
    \Upsilon_2(\eta(x'_{\Lambda_2})) &=  \eta\left( \int_{\mathcal{G}_2}  \langle \eta( x'_{\Lambda_2} ),  \eta(g(t_2))  \rangle  dg  \right)
\end{align}
Here the receptive field of $x'_{\Lambda_2}$ is the union of the receptive fields of the four $\Upsilon_{1i}$. Since layer 1 features are invariant to the respective groups, variation in $x'_{\Lambda_2}$ occurs only due to transformations that do not fall into the groups modelled by layer 1 nodes. However, in order for group integration to be applied at layer 2, the transformation needs to be propagated or be covariant  in the layer 1 feature space. We express this formally through a property which was previously shown to be true for hierarchical architectures employing group integrals \cite{anselmi2013unsupervised}.

\begin{property}\label{conjecture_1}
(Covariance in the TN node feature space) Given a unitary $g_{\Lambda_2}$ over $\Lambda_2$ that is not modelled by the unitary groups in layer 1 i.e. $\{\mathcal{G}_{11}, \mathcal{G}_{12}, \mathcal{G}_{13}, \mathcal{G}_{14}   \}$, $\exists g'_{\Lambda_2}$ s.t. $$   \left[ \begin{array}{c} I( g_{\Lambda_2|\Lambda_{11}} (x_{\Lambda_{11}})) \\ I(g_{\Lambda_2|\Lambda_{12}} (x_{\Lambda_{12}})) \\ I(g_{\Lambda_2|\Lambda_{13}} (x_{\Lambda_{13}})) \\ I(g_{\Lambda_2|\Lambda_{14}} (x_{\Lambda_{14}}))  \end{array} \right]  =  g'_{\Lambda_2} (x'_{\Lambda_2} ) $$where $ g_{\Lambda_a|\Lambda_{b}} $ is the transformation $g_{\Lambda_a}$ restricted to the support $\Lambda_b$ and $ [ I(x_{\Lambda_{11}}),  I(x_{\Lambda_{12}}), I(x_{\Lambda_{13}}), I(x_{\Lambda_{14}}) ]^T  = x'_{\Lambda_2}$
\end{property}

This property allows for a unitary transformation acting on the support $\Lambda_2$ in the input space to have a corresponding action or effect in the feature space. For instance, if one considers an in-plane rotation over a $16\times 16$ image, then a $2\times 2$ pooling (which is essentially feature extraction with the identity template) of pixels still preserves the rotation to a large degree. Feature extraction with general templates will also preserve the transformation due to the linearity of the dot-product.



Applying Property~\ref{conjecture_1} to the layer 2 features, we have for any transformed $g'_{\Lambda_2}(x'_{\Lambda_2})$ having support over $\Lambda_2$, 
\begin{align}
    &\Upsilon_2(\eta \circ g'_{\Lambda_2}( x'_{\Lambda_2}) )\\ 
    &=  \eta\left( \int_{\mathcal{G}_2}  \langle \eta \circ  g'_{\Lambda_2}  (x'_{\Lambda_2} ),  \eta \circ g(t_2)  \rangle  dg  \right)\\
    &=  \eta\left( \int_{\mathcal{G}_{\eta 2}}  \langle g'_{\eta\Lambda_2} \circ  \eta(  x'_{\Lambda_2} ),  g_\eta \circ  \eta(t_2)  \rangle  dg_{\eta}  \right) \label{eq_2_proof}\\
    &=  \eta\left( \int_{\mathcal{G}_{\eta 2}}  \langle  \eta(  x'_{\Lambda_2} ),  (g'_{\eta\Lambda_2})^{-1}\circ  g_\eta \circ  \eta(t_2)  \rangle  dg_{\eta}   \right)\\ 
     &=  \eta\left( \int_{\mathcal{G}_{\eta 2}}  \langle  \eta(  x'_{\Lambda_2} ),   g_\eta\circ \eta(t_2)  \rangle  dg_{\eta}   \right)\label{eq_4_proof} \\ 
    &= \Upsilon_2(\eta(x'_{\Lambda_2}))
\end{align}
Equation~\ref{eq_2_proof} utilizes Theorem~\ref{theorem_1} and Equation~\ref{eq_4_proof} utilizes the fact that since the templates are considered to be modelled using the same transformation model as the raw input images due to training (thereby layer 2 always observes the same group structure $(g'_{\eta\Lambda_2})^{-1} \in \mathcal{G}_{\eta 2}$). This implies that $(g'_{\eta\Lambda_2})^{-1} g_\eta$ forms a bijective mapping to  some $g'_\eta \in \mathcal{G}_{\eta 2}$ and thus the transformation results in a rotation of the group elements or a reordering of the group\footnote{Here the identity element maps to $(g'_{\eta\Lambda_2})^{-1}$ and $g'_{\eta\Lambda_2}$ maps to the identity}. The group essentially is invariant and hence the group integral does not change. We therefore arrive at the layer 2 \textit{linear} invariance expression, which is
\begin{align}
    \Upsilon_2(\eta(g_{\Lambda_2} x'_{\Lambda_2}) ) = \Upsilon_2(\eta(x'_{\Lambda_2})) \label{eq_non_linear_inv}
\end{align}
Here $x'_{\Lambda_2}$ is simply the activation response or output of the layer 1 TN nodes. Intuitively, \textit{ layer 2 TN node is invariant to linear transformations over the larger support $\Lambda_2$ in the input space}. The invariance expression of $\Upsilon_2$ however, is coupled with the non-linearity $\eta$ from the previous layer (\emph{i.e.} layer 1). This coupling is what allows the node to model more general non-linear invariances as we will see shortly. In order to highlight the invariance specifically, we rewrite the invariance expression for a general $x$ and unitary group element $g$ and replacing $\Upsilon_2(\eta(\cdot)) = \Gamma_2(\cdot)$ where with 2 denotes layer 2 node. Therefore, we have $$\Gamma_2(g (x)) = \Gamma_2(x)$$  Further, it is interesting to note that if $x$ itself is a non-linear transformation of some $x'_{\Lambda_2}$, \emph{i.e.} $x = \eta \circ g'_{\Lambda_2} (x'_{\Lambda_2})$, we then arrive at our main result.



\begin{theorem}\label{theorem_non_linear_inv}
(Two-layer TN node Non-linear Invariance) Under a unitary group $\mathcal{G}_{\Lambda_2}$ acting on the support $\Lambda_2$, the output of the second layer node $\Upsilon_2(\eta(\cdot)) = \Gamma_2(\cdot)$ covering the support $\Lambda_2$, is invariant to the action or transformations of $\eta \circ  \mathcal{G}_{\Lambda_2}$ on any input $x'_{\Lambda_2}$, \emph{i.e.} $$\Gamma_2(x'_{\Lambda_2}) = \Gamma_2(\eta \circ g'_{\Lambda_2}(x'_{\Lambda_2})) ~~\forall g'_{\Lambda_2} \in \mathcal{G}_{\Lambda_2}, \forall x'_{\Lambda_2}$$ for $\eta(\cdot)$ satisfying the conditions of unitarity and stability.
\end{theorem}
\begin{proof} We have,
\begin{align}
    \Gamma_2(\eta\circ g'_{\Lambda_2} (x'_{\Lambda_2})) &=  \Upsilon_2(\eta \circ  \eta \circ g'_{\Lambda_2}(x'_{\Lambda_2}  ))\\ &= \Upsilon_2( \eta\circ g'_{\Lambda_2}(x'_{\Lambda_2})  )\\ &= \Upsilon_2( \eta(x'_{\Lambda_2})  ) = \Gamma_2(x'_{\Lambda_2})
\end{align}
The second equality utilizes the stability property of $\eta(\cdot)$ whereas the third equality arises from the invariance of $\Upsilon(\eta(\cdot))$ as demonstrated in Equation~\ref{eq_non_linear_inv}.
\end{proof}

This shows that a TN node at layer 2 is invariant to a non-linear transformation $\eta\circ  g'_{\Lambda_2}$ of any $x'_{\Lambda_2}$ over the support $\Lambda_2$. Combining the invariance generated due to the first layer as well, the node overall is invariant to the non-linear transformation $g_{\Lambda_{11}}\circ g_{\Lambda_{12}}\circ g_{\Lambda_{13}}\circ g_{\Lambda_{14}}\circ \eta \circ  g_{\Lambda_{2}}(\cdot)$. More specifically, the four layer 1 TN nodes are invariant to the first 4 group elements in the sequence and the second node is invariant to the last element combined with the non-linearity. This result can be extended to multiple layers directly.

\textbf{Rich non-linear invariance in the case of a multi-layered TN:} Our result can be naturally extended  to multi-layered TNs. Consider a TN with $L$ layers with $k_l$ non-overlapping receptive fields at layer $l$. The $i^{th}$ node at layer $l$ observes a receptive field $\Lambda_{li}$. Assuming the last layer $L$'s receptive field covers the entire image, the $L^{th}$ layer node is invariant to the following class of transformation $\mathcal{G}_L= g_{\Lambda_{11}}\circ ...\circ g_{\Lambda_{1k_1}}\circ \eta \circ   g_{\Lambda_{l1}}\circ ...\circ g_{\Lambda_{lk_l}}\circ \eta \circ g_{\Lambda_{L}}\circ (\cdot)$. One can rewrite the form as  $\mathcal{G}_L= g_{\Lambda_{1}} \eta \circ ...  g_{\Lambda_{l}} ...\circ \eta \circ g_{\Lambda_{L}}\circ (\cdot)$ where we collapse all unitary transforms in a layer into one variable. This class of transformations contains $L-1$ non-linearities and is extremely rich. The \textit{structure} of a TN itself along with unitary group modelling and a special class of non-linearities allow for generating invariance to such a large class of transformations of the input.

\textbf{Hierarchy helps in efficient invariance generation:} Consider the class of transformations  $h(x) =  g_1\circ \eta \circ g_2 ...\eta \circ g_L(x)$ that a $L$ layered ConvNet is invariant to. Using a naive single layered approach to be invariant to $h(x)$, one would need to generate all transforms modelled by $h$ and integrate over them. If $g_i\in \mathcal{G}_i$ for a finite group $\mathcal{G}_i$ with the cardinality  $|\mathcal{G}_i|$, then the size of all possible $h(x)$ is of the form $\prod_{j=1}^L |\mathcal{G}_j| $. If all individual groups  have the same cardinality $|\mathcal{G}|$, then the number of transformations is of the order $|\mathcal{G}|^L$. However, with a hierarchical architecture that generates invariance to the individual groups at every layer, the machine only needs to integrate over $|\mathcal{G}_i|$ transforms at layer $i$.  The total number of transforms needed to be integrated over becomes $\sum_{j=1}^L |\mathcal{G}_j|$. Under the assumption that all groups are the same size, the total number becomes $L|\mathcal{G}|$. This is a significant reduction from $\mathcal{O}(|\mathcal{G}|^L)$ to $\mathcal{O}(|\mathcal{G}|)$, by an order of $L$. Even though deeper networks require more data to train well, they can generate invariance to more complicatedly transformed data more efficiently. Further, lower layers having a smaller receptive field helps since cardinality of the transformation groups acting on smaller sized input is lower than those for a larger sized input. This helps the network in factorizing transformations with smaller less complicated transformations before deadling with larger more complicated non-linear ones.

\subsection{The need for multiple templates or channels} 

Up until now in our analysis, we have assumed that the TN nodes have a single channel or a single template. The  feature at layer 2 and above was multi-dimensional merely due to the distinct support sets $\Lambda$ over the image. Our results however extend naturally to multiple channels with multiple templates since we make no assumption regarding the relation between the templates. \cite{anselmi2013unsupervised} suggest the need for multiple templates as a way of better measuring the invariant probability distribution (over pixels) of a group of transformed images. Indeed, the quantity $\langle g(x), t   \rangle~~\forall g\in \mathcal{G}$ is a 1-D projection along $t$ of the distribution of the set $\{ g(x) \}~~\forall g \in \mathcal{G}$. More the number of templates or channels, better the estimate of the probability distribution due to the Cramer-Wold's theorem. This result also holds true for our framework since the dot-products in a TN are 1-D projections of the transformed data onto a TN node template. The reason that this probability distribution is important is because \cite{anselmi2013unsupervised} show that it itself is an invariant to the action of the group $\mathcal{G}$. Therefore, moments of the distribution are also invariant including the first (leading to mean pooling) and infinite moment (leading to max pooling). Group integrals can be seen as measuring the first moment. Thus our results can be integrated with theirs almost seamlessly.









\section{Towards Convolutional Architectures}

Our framework for Transformation Networks models the transforming templates in each TN node as unitary groups. In order to apply supervised learning or back-propagation to these architectures, one must address the crucial issue of maintaining group structure in the template sets while optimizing the templates themselves. If back-propagation is applied naively to all templates in a template set, assuming they start with a group structure intact, they will converge to the same template throughout the set and the group structure will be lost. One way of addressing this issue is to assume all groups in the TN to be \textit{identical} and \textit{parametric}. A parametric transformation that can be efficiently applied would allow us to explicitly generate the template set on the fly for pooling or group integration. In doing so, back-propagation needs to only update one of the templates in each template set. This is because the group structure is explicitly maintained by applying the parametric transformations to that template to generate rest of the template set. ConvNets adopt this exact approach with the transformation of choice being translation since translations can be efficiently implemented during run-time as convolutions. 

Our results apply directly to ConvNets since they are simply TNs instantiated with the unitary groups being discrete translations. We therefore find that the \textit{architecture} of a ConvNet itself allows it to be able to model non-linear transformations. The weight sharing property of ConvNets (leading to convolutions), originally meant for regularization or merely local translation invariance, therefore has a very powerful \textit{by-product of generating invariance to much more complicated non-linear transforms overall}. \cite{goodfellow2009measuring} has studied the problem of visualizing and measuring these invariances generated by a ConvNet and provided empirical justifications for increasing depth. Recall that given a ConvNet with $L$ layers, it is invariant to a class of transformations with at least $L-1$ non-linearities. By increasing the depth of a ConvNet, we are essentially adding in a layer of non-lnearity in the class of transformations that the ConvNet can be invariant towards.  This provides a theoretical justification for the well known fact that depth can improve performance of a ConvNet.

\section{Conclusion}

We have shown that TNs (and thereby ConvNets) are can be invariant to non-linear transformations of the input despite pooling over mere local unitary transformations. We also showed that deeper networks are able to model much richer classes of transformations. Further, we find that a hierarchical architecture allows the network to generate invariance much more efficiently than a non-hierarchical network. 




\bibliography{iclr2017_conference}

\begin{thebibliography}{14}
\providecommand{\natexlab}[1]{#1}
\providecommand{\url}[1]{\texttt{#1}}
\expandafter\ifx\csname urlstyle\endcsname\relax
  \providecommand{\doi}[1]{doi: #1}\else
  \providecommand{\doi}{doi: \begingroup \urlstyle{rm}\Url}\fi

\bibitem[Anselmi et~al.(2013)Anselmi, Leibo, Rosasco, Mutch, Tacchetti, and
  Poggio]{anselmi2013unsupervised}
Fabio Anselmi, Joel~Z Leibo, Lorenzo Rosasco, Jim Mutch, Andrea Tacchetti, and
  Tomaso Poggio.
\newblock Unsupervised learning of invariant representations in hierarchical
  architectures.
\newblock \emph{arXiv preprint arXiv:1311.4158}, 2013.

\bibitem[Bruna et~al.(2013)Bruna, Szlam, and LeCun]{bruna2013learning}
Joan Bruna, Arthur Szlam, and Yann LeCun.
\newblock Learning stable group invariant representations with convolutional
  networks.
\newblock \emph{arXiv preprint arXiv:1301.3537}, 2013.

\bibitem[Cohen et~al.(2015)Cohen, Sharir, and Shashua]{CohenSS15a}
Nadav Cohen, Or~Sharir, and Amnon Shashua.
\newblock On the expressive power of deep learning: {A} tensor analysis.
\newblock \emph{CoRR}, abs/1509.05009, 2015.
\newblock URL \url{http://arxiv.org/abs/1509.05009}.

\bibitem[Delalleau \& Bengio(2011)Delalleau and Bengio]{delalleau2011shallow}
Olivier Delalleau and Yoshua Bengio.
\newblock Shallow vs. deep sum-product networks.
\newblock In \emph{Advances in Neural Information Processing Systems}, pp.\
  666--674, 2011.

\bibitem[Goodfellow et~al.(2009)Goodfellow, Lee, Le, Saxe, and
  Ng]{goodfellow2009measuring}
Ian Goodfellow, Honglak Lee, Quoc~V Le, Andrew Saxe, and Andrew~Y Ng.
\newblock Measuring invariances in deep networks.
\newblock In \emph{Advances in neural information processing systems}, pp.\
  646--654, 2009.

\bibitem[Haeffele \& Vidal(2015)Haeffele and Vidal]{HaeffeleV15}
Benjamin~D. Haeffele and Ren{\'{e}} Vidal.
\newblock Global optimality in tensor factorization, deep learning, and beyond.
\newblock \emph{CoRR}, abs/1506.07540, 2015.
\newblock URL \url{http://arxiv.org/abs/1506.07540}.

\bibitem[Kawaguchi(2016)]{kawaguchi2016deep}
Kenji Kawaguchi.
\newblock Deep learning without poor local minima.
\newblock In \emph{Advances In Neural Information Processing Systems}, pp.\
  586--594, 2016.

\bibitem[Krizhevsky et~al.(2012)Krizhevsky, Sutskever, and
  Hinton]{krizhevsky2012imagenet}
Alex Krizhevsky, Ilya Sutskever, and Geoffrey~E Hinton.
\newblock Imagenet classification with deep convolutional neural networks.
\newblock In \emph{Advances in neural information processing systems}, pp.\
  1097--1105, 2012.

\bibitem[LeCun et~al.(1998)LeCun, Bottou, Bengio, and
  Haffner]{lecun1998gradient}
Yann LeCun, L{\'e}on Bottou, Yoshua Bengio, and Patrick Haffner.
\newblock Gradient-based learning applied to document recognition.
\newblock \emph{Proceedings of the IEEE}, 86\penalty0 (11):\penalty0
  2278--2324, 1998.

\bibitem[Mallat(2012)]{mallat2012group}
St{\'e}phane Mallat.
\newblock Group invariant scattering.
\newblock \emph{Communications on Pure and Applied Mathematics}, 65\penalty0
  (10):\penalty0 1331--1398, 2012.

\bibitem[Martens \& Medabalimi(2014)Martens and
  Medabalimi]{martens2014expressive}
James Martens and Venkatesh Medabalimi.
\newblock On the expressive efficiency of sum product networks.
\newblock \emph{arXiv preprint arXiv:1411.7717}, 2014.

\bibitem[Nair \& Hinton(2010)Nair and Hinton]{nair2010rectified}
Vinod Nair and Geoffrey~E Hinton.
\newblock Rectified linear units improve restricted boltzmann machines.
\newblock In \emph{Proceedings of the 27th international conference on machine
  learning (ICML-10)}, pp.\  807--814, 2010.

\bibitem[Paul \& Venkatasubramanian(2014)Paul and
  Venkatasubramanian]{paul2014does}
Arnab Paul and Suresh Venkatasubramanian.
\newblock Why does deep learning work?-a perspective from group theory.
\newblock \emph{arXiv preprint arXiv:1412.6621}, 2014.

\bibitem[Rossius et~al.(1998)Rossius, Zenker, Ittner, and
  Dilger]{rossius1998short}
Rolf Rossius, G{\'e}rard Zenker, Andreas Ittner, and Werner Dilger.
\newblock A short note about the application of polynomial kernels with
  fractional degree in support vector learning.
\newblock In \emph{European Conference on Machine Learning}, pp.\  143--148.
  Springer, 1998.

\end{thebibliography}
\bibliographystyle{iclr2017_conference}

\section{Proofs of theoretical results}

All group theoretic resuts hold true for finite groups as well.

\subsection{Proof of Lemma 3.1}
\begin{proof}We have,
$$g' \left(\int_\mathcal{G} g (x) ~dg\right) = \int_\mathcal{G} g'\circ g (x) ~dg =  \int_\mathcal{G} g''(x) ~dg'' = \int_\mathcal{G} g (x) ~dg$$

Since the normalized Haar measure is invariant, \emph{i.e.} $dg = dg'$. Intuitively, $g'$ simply rearranges the group integral owing to elementary group properties.
\end{proof}

\subsection{Proof of Lemma 3.2}
\begin{proof}We have,
 \begin{align}
 \Upsilon(g'(x)) &= \eta\left(  \int_{\mathcal{G}} \langle  g'(x), g(t)  \rangle dg   \right)\label{eq_1}  \\
 &= \eta\left(  \int_{\mathcal{G}} \langle  x, g'^{-1}(g(t))  \rangle dg   \right)\label{eq_2}\\
 &= \eta\left(  \int_{\mathcal{G}} \langle  x, g''(t)  \rangle dg''   \right)\label{eq_3}\\
 &=  \Upsilon(x)
 \end{align}
 
 Eq.~\ref{eq_2} uses the fact that $g'\in \mathcal{G}$ is unitary.  Eq.~\ref{eq_3} showcases a change of variable. Since $g', g\in \mathcal{G}$, therefore $g''\in \mathcal{G}$. Further $dg=dg''$ since the Haar measure is unitary.
 
\end{proof}

\subsection{Proof of Theorem 3.3}
\begin{proof}
We have $\langle \eta(gx), \eta(gy)\rangle = \langle \eta(x), \eta(y)\rangle = \langle g_\eta \eta(x), g_\eta \eta(y)\rangle$, since the function $\eta$ is unitary. We define $g_\eta \eta(x)$ as the action or transformation of $g_\eta $ on $\eta(x)$. This is one of the requirements of a unitary operator, however $g_\eta$ needs to be linear. Linearity of $g_\eta $ can be derived from the linearity of the inner product and its preservation under $g_\eta$ in $\eta$. For an arbitrary vector $p$ and a scalar $\alpha$, we have 
\begin{align}
&|| \alpha g_\eta(p) -  g_\eta (\alpha p)||^2 \\
&= \langle  \alpha g_\eta p -  g_\eta (\alpha p),  \alpha g_\eta p -  g_\eta (\alpha p)  \rangle\\
&= ||\alpha g_\eta(p) ||^2 + || g_\eta (\alpha p)||^2  - 2 \langle \alpha g_\eta(p) ,  g_\eta (\alpha p)  \rangle\\
&=  |\alpha|||p ||^2 + || \alpha p||^2  - 2 \alpha^2 \langle  p, p  \rangle = 0
\end{align}
Similarly for vectors $p, q$, we have $ || g_\eta(p+q) -  (g_\eta (p)+g_\eta (q))||^2 = 0$

We now prove that the set $\mathcal{G}_\eta$ is a group. We start with proving the closure property. We have for any fixed $g_\eta, g'_\eta\in \mathcal{G}_\eta$
$$ g_\eta (g'_\eta(\eta(x))) = g_\eta(\eta(g'(x))) = \eta(g(g'(x))) = \eta(g''(x))  = 
g''_\eta(\eta(x))$$

Since $g''\in \mathcal{G}$ therefore $g''_\eta\in \mathcal{G}_\eta$ by definition. Also, $ g_\eta g'_\eta =  g''_\eta$ and thus closure is established. Associativity, identity and inverse properties can be proved similarly. The set $\mathcal{G}_\eta = \{ g_\eta ~|~ g_\eta: \eta(x)\rightarrow \eta(gx)~\forall g\in \mathcal{G} \}$ is therefore a unitary-group in $\eta$.
\end{proof}

\end{document}